%% file: main_bmvc.tex
\documentclass{bmvc2k}
\pdfoutput=1
\PassOptionsToPackage{table,x11names}{xcolor}
\PassOptionsToPackage{numbers,compress}{natbib}

\usepackage{caption}
\usepackage{subfigure}

\usepackage{booktabs} 
\usepackage{graphicx}
\usepackage[utf8]{inputenc} 
\usepackage[T1]{fontenc}    

\usepackage{hyperref}            
\usepackage{amsthm}
\usepackage{nicefrac}    
\usepackage{multirow}
\usepackage{makecell}
\usepackage{etoolbox}
\usepackage{enumitem}
\usepackage[hang,flushmargin]{footmisc} 
\usepackage{graphics}
\usepackage{amsmath,amssymb,mathtools,bm}
\usepackage{physics} 
\usepackage{siunitx} 
\usepackage{placeins}
\usepackage[outline]{contour}
\usepackage{cleveref}
\usepackage[utf8]{inputenc} 
\usepackage[T1]{fontenc}    
\usepackage{url}              
\usepackage{amsfonts}     
\usepackage{nicefrac} 
\usepackage{microtype}     
\usepackage{adjustbox}
\newcolumntype{x}[1]{>{\centering\arraybackslash}p{#1}}
\usepackage{array, booktabs, boldline} %
\usepackage{float}

\usepackage{wrapfig,lipsum}

\newtheorem{theorem}{Theorem}
\newtheorem*{theorem*}{Theorem}
\newtheorem{lemma}{Lemma}
\newtheorem*{lemma*}{Lemma}

\title{ \hspace{0.25cm}  Exploring the Limits of Deep Image \\ \hspace{0.25cm}  Clustering using Pretrained Models}

\def\thefootnote{*}\footnotetext{Equal contribution.}

\addauthor{Nikolas Adaloglou\thefootnote}{adaloglo@hhu.de}{1}
\addauthor{Felix Michels\thefootnote}{ felix.michels@hhu.de}{1}
\addauthor{Hamza Kalisch}{  hakal104@hhu.de}{1}
\addauthor{Markus Kollmann}{ markus.kollmann@hhu.de}{1}

\addinstitution{
 Heinrich Heine University \\
 Düsseldorf, Germany
}

\runninghead{Adaloglou et al.}{Exploring the Limits of Deep Image Clustering}

\begin{document}

\maketitle

\begin{abstract}
We present a general methodology that learns to classify images without labels by leveraging pretrained feature extractors. Our approach involves self-distillation training of clustering heads, based on the fact that nearest neighbours in the pretrained feature space are likely to share the same label. We propose a novel objective that learns associations between image features by introducing a variant of pointwise mutual information together with instance weighting. We demonstrate that the proposed objective is able to attenuate the effect of false positive pairs while efficiently exploiting the structure in the pretrained feature space. As a result, we improve the clustering accuracy over $k$-means on $17$ different pretrained models by $6.1$\% and $12.2$\% on ImageNet and CIFAR100, respectively. Finally, using self-supervised vision transformers we achieve a clustering accuracy of $61.6$\% on ImageNet. The code is available at \url{https://github.com/HHU-MMBS/TEMI-official-BMVC2023}.
\end{abstract}

\section{Introduction}
\label{intro}
Given a plethora of publicly available pretrained vision models, we ask the following questions: a) how well-structured is the feature space of pretrained architectures with respect to label-related information, and b) how to best adapt this structure to unsupervised tasks. To answer these questions, we focus on unsupervised image classification, also known as image clustering. Image clustering is the task of assigning a semantic label to an image, given an a priori finite set of classes. Ultimately, image clustering consists of simultaneously learning the relevant representations and the cluster assignments.

To begin addressing the aforementioned questions, we present the key challenges regarding image clustering. First, given that we can roughly estimate the number of ground-truth labels, the underlying distribution among classes is hard to infer from the data, which is typically assumed to be uniform. Second, the images should ideally be classified consistently (images of the same class are grouped together) and confidently (one-hot prediction probability).
\newpage
Consistency can be achieved by either learning features that are invariant under transformations of the same image (e.g.\ cropping, colour jitter, etc.) or invariant w.r.t.\ to substitution by other images that belong to the same semantic class.
Consequently, clustering methods are generally prone to degenerate solutions \cite{sscn}. In other words, samples tend to collapse into a single cluster, or the prediction probability spreads out uniformly.

\input{figures/scatter_plot_kmeans.tex}

It is well-established that representation learning plays a critical role in image clustering \cite{chang2017deep}, which is achieved with self-supervised learning \cite{simclr, moco,zbontar2021barlow,byol}. Recent studies have demonstrated that self-supervised features are typically more transferable to new tasks than features from supervised learning \cite{ericsson2021ssltranfer}. The frequently used joint-embedding architectures \cite{byol,dino} are by design invariant to strong image transformations that preserve label-related information. That renders these architectures as promising candidates for image clustering \cite{scan}, which has not been thoroughly explored at scale \cite{tsp}. Even though self-supervision \cite{ericsson2021ssltranfer} and vision transformers (ViTs) \cite{naseer2021intriguing} have been separately established for representation learning, limited research has been conducted to study self-supervised ViTs or vision-language models (i.e.\ CLIP \cite{radford2021clip}).

How to adapt a pretrained model for image clustering is non-trivial. For instance, it is well known that $k$-means is sub-optimal, as it often leads to imbalanced clusters \cite{scan} since it is primarily suitable for evenly scattered data samples around their centroids \cite{yang2017towards}. On the other hand, deep image clustering methods normally rely on pairs by mining the nearest neighbours (NN) based on their feature similarity \cite{dwibedi2021little,huang2019and}. Still, images that are close in the feature space do not always share the same semantic class \cite{scan} and, therefore, must be considered as noisy pairs. 

In this paper, a two-stage method that extends the existing multi-stage clustering approaches is proposed. In contrast to \cite{scan}, where features are learned from scratch for each dataset, we show that multi-stage clustering approaches can leverage pretrained models trained on larger-scale datasets (Fig.\ 1) and focus on learning the cluster assignments. To this end, a self-distillation clustering framework is introduced using a novel objective based on pointwise mutual information and instance weighting. Second, a comprehensive experimental study across models and datasets is conducted. Therein, we report an average gain of $6.1$\% and $12.2$\% in clustering accuracy compared to $k$-means on ImageNet and CIFAR100 across $17$ pretrained models, as illustrated in \Cref{fig:main}. Overall, we show that ViTs capture the most transferable label-related features. Finally, we find that self-supervised ViTs \cite{msn} achieve state-of-the-art results ($61.6$\% clustering accuracy) on ImageNet without using the ground-truth labels or external data.

\section{Related Work}
\label{related}
\noindent\textbf{Single-stage Deep Image Clustering Methods.} Deep image clustering approaches can be roughly divided into single and multi-stage methods. The majority of single-stage methods alternate between learning the features and the clusters, i.e.\ in an expectation-maximization manner. For instance, in DAC, \cite{chang2017deep} formulate a binary pairwise-classification task, where at each iteration, pairs are selected based on their feature similarity. Next, the computed pairs are used to train a convolutional neural network (CNN). In the same direction, in DeepCluster \cite{deepcluster}, the authors alternate between clustering the features of a CNN with $k$-means \cite{lloyd1982kmeans} and using the obtained cluster assignments as pseudo-labels to optimize the parameters of the CNN. 

Later on, in \cite{sela}, the authors demonstrate that DeepCluster is prone to degenerate solutions that are avoided via particular hyperparameter choices. To that end, the authors design a multi-step pseudo-label extraction framework called SeLa. The latter iteratively estimates the pseudo-label assignment matrix under the equipartition constraint. In PCL \cite{pcl}, the authors formulate clustering as learning the cluster centroids with $k$-means in parallel with optimizing the network via contrastive learning \cite{simclr}. To overcome the class collision of the negative pairs, \cite{propos} extend PCL in a proximal framework called ProPos. ProPos only maximizes the distance between the cluster centroids with contrastive learning while mining NN for neighbouring sample alignment \cite{byol}. However, most of the existing approaches still rely on $k$-means for estimating the clusters (pseudo-labels).

Several single-stage approaches exist, which aim to learn the feature representations and clusters jointly. Single-step methods are known to be sensitive to weight initialization \cite{dang2021nearest}. In this direction, DCMM is developed \cite{wu2019deep} to progressively mine NN in the feature space as well as high-confident samples. Another single-stage end-to-end example is IIC, wherein \cite{ji2019invariant} derives a mutual information-based objective for paired data to train a CNN. Nevertheless, the aforementioned approaches only consider stochastic transforms of the same image to obtain a pair. They are hence limited to solely learning invariances w.r.t.\ image augmentations, which cannot cover the variability of a given class \cite{dwibedi2021little}. More recently, \cite{sscn} presented a single-stage end-to-end method, called SSCN, that employs a variant of the cross-entropy loss.

\noindent\textbf{Multi-stage Deep Image Clustering Methods}.
Multi-stage methods initially design a pretext task in order to learn semantically meaningful features, such as denoising autoencoders \cite{xie2016unsupervised}. A major breakthrough in deep image clustering is established by the adoption of contrastive learning \cite{simclr,moco}. For instance, \cite{scan} decouples image clustering into three distinct steps, starting with contrastive learning. Subsequently, the authors train a head to cluster the mined NN from the extracted features. Lastly, they use the pseudo-labels from the confidently assigned samples to fine-tune the whole architecture. A similar approach, called NNM \cite{dang2021nearest}, aims to mine NN from the batch and dataset features. 

Recently, \cite{tsp} leverage self-supervised pretrained ViTs \cite{dino} and train a clustering head, which is closer to our method. Nevertheless, their approach (TSP) heavily relies on $k$-means for the weight initialization phase. Surprisingly, very few image clustering approaches \cite{scan,sscn,sela} have been successfully applied on ImageNet \cite{deng2009imagenet}. Besides, most methods report results only with Resnet50 \cite{resnet}, while superior architectures for image recognition remain unexplored \cite{convnext,vit}.

\section{Proposed Method}
\label{method}

\newcommand{\pmi}{\operatorname{pmi}}
\newcommand{\pmib}{\pmi^\beta}
\newcommand{\bpmi}{\pmi^\beta}
\newcommand{\KL}[2]{\operatorname{KL}(#1 \parallel #2)}
\newcommand{\MI}[2]{\operatorname{I}(#1;#2)}
\newcommand{\Entropy}{\operatorname{H}}
\newcommand{\Expected}[2]{\operatorname{\mathbb{E}}_{#1} [#2]}
\newcommand{\pthet}{q}
\newcommand{\popt}{\pthet^*}
\newcommand{\pstud}{q_s}
\newcommand{\pteach}{q_t}
\newcommand{\probEma}{\operatorname{EMA}}

\subsection{Classification Model}
\label{subsec:mi-objective}
Our aim is to learn a probabilistic classifier from pairs of examples that share label-related information. 
We assume that the data distribution, $p(x)$, is the result of a generative process, $c \sim p(c)$ and $x \sim p(x\vert c)$, with $p(c)$ the prior probability that an example belongs to a class $c \in \{1,..,C\}$. Consequently, the joint distribution, $p(x,x')$ that a pair of examples, $(x,x')$, belongs to the same class is given by $ p(x,x') = \sum_{c=1}^C p(x|c)p(x'|c)p(c).$
We introduce a parametrized probabilistic classifier, $q(c\vert x)$, that distributes examples $x\sim p(x)$ among classes, with class occupancy given by $q(c)=\Expected{x \sim p(x)}{q(c|x)}$. Using Bayes' theorem $q(x|c) = q(c|x)p(x) / q(c)$, the joint distribution, $p(x,x')$, can be approximated by 
\begin{equation}
    \label{eq:est-bottleneck}
    q(x,x') = \sum_{c=1}^C q(x|c)q(x'|c)q(c).
\end{equation}
To estimate the association between $x$ and $x'$ we introduce the pointwise mutual information, $\pmi(x,x')$ \cite{church1990word}, defined by
\begin{equation}
   \label{eq:pmi-definition}
\pmi(x,x') \coloneqq \log \frac{q(x,x')}{p(x)p(x')} \\
            = \log \sum_{c=1}^C \frac{q(c|x) q(c|x')}{q(c)}. 
\end{equation}
 
\begin{theorem}
\label{thm:optimal}
    If (i) each example $x\sim p(x)$ belongs to one and only one cluster under the generative model $p(x)= \sum_c p(x|c)p(c)$, (ii) the joint distribution $p(x,x')$ is known, and (iii) $\popt(c|x)$ is a probabilistic classifier defined by
    \begin{equation}
    \label{eq:thm-objective}
    \popt(c|x) = \arg\max_{q(c|x)} \Expected{x,x' \sim p(x,x')}{\pmi(x,x')},
    \end{equation}
    then $\popt(c|x)$ is equal to the optimal probabilistic classifier, $p(c|x)=p(x|c)p(c)/p(x)$, up to a permutation of cluster indices.
\end{theorem}
The proof can be found in the supplementary material. \Cref{thm:optimal} states that under condition (i) the knowledge of pairs of examples belonging to the same class suffices to establish an objective for an optimal classification model.

\subsection{Self-distillation Clustering Framework} 
The starting point is a pretrained feature extractor (backbone) $g(\cdot)$ that assigns each example $x$ in the dataset $D$ a feature vector $g(x)$. We mine the $k$ nearest neighbours ($k$-NN) of $x$ in the feature space by computing the cosine similarity between $g(x)$ and the feature vectors of all other images in $D$. We denote the set of $k$-NN for x by $S_{x}$. %
During training, we randomly sample $x$ from $D$ along with $x'$ from $S_x$, to generate image pairs that share label information with high probability.

We introduce two clustering heads, a \emph{student head}, $h_s(\cdot)$, and a \emph{teacher head}, $h_t(\cdot)$, that share the same architecture but differ w.r.t.\ their parameters, $\theta_s$, and $\theta_t$. Each head consists of a three-layer fully connected feed-forward network. The image pairs $x,x'$ are fed to the shared backbone and subsequently, in the two heads, $h_s(g(x))$ and $h_t(g(x'))$. %
The head outputs are converted to probabilistic classifiers, $\pstud(c|x)$ and $\pteach(c|x')$, using a temperature-scaled softmax function, which for the student's head is given by
\begin{equation}
    \pstud(c|x) = \frac
    {\exp(h_s(g(x))_c \,/\, \tau)} 
    {\sum_{c'} \exp(h_s(g(x))_{c'} \,/\, \tau)},
\end{equation}
where $\tau$ is the temperature hyperparameter. Unlike previous self-distillation frameworks \cite{dino}, we use the same temperature $\tau=0.1$ for both heads. We approximate the pointwise mutual information by 
\begin{equation}
\widetilde\pmi(x,x')\coloneqq
     \log \sum_{c=1}^C \frac{
    \pstud(c|x)\pteach(c|x')
    }{\tilde q_t(c)}.
    \label{eq:objective}
\end{equation}
and estimate $q(c)$ by an exponential moving average (EMA) over batches using the teacher's head
\begin{equation}
\tilde q_t(c) \leftarrow m \,\tilde q_t(c) +
(1 - m) \frac{1}{B}\sum_{i=1}^B \pteach(c|x_i),
\end{equation}
with $B$ the batch size and $m \in (0,1)$ a momentum parameter. In practice, we symmetrize \Cref{eq:objective} to compute the loss function
\begin{equation}
    \mathcal{L}(x,x') \coloneqq -\frac{1}{2}\left( \widetilde\pmi(x,x')+ \widetilde\pmi(x',x)\right).
    \label{eq:symmetrized-loss}
\end{equation}
Note that only the parameters $\theta_s$ of the student's head are updated using backpropagation. The parameters of the teacher's head, $\theta_t$, are updated by an exponential moving average for the student parameters, $\theta_s$, over past update steps \cite{dino,byol}. As a result, $p_t(c|x)$ represents a sufficiently stable target distribution for the student head. In contrast to other self-distillation frameworks \cite{dino}, no complicated adaptation of softmax temperatures over training is required. Following previous work \cite{scan}, we employ an ensemble of $H$ independent clustering heads in training (\cref{fig:diagram}), which alleviates the variability stemming from random initialization. For the evaluation, we use the teacher head with the lowest training loss.

\subsection{Balancing class utilization}
For a dataset $D$ that has been generated using balanced classes, $p(c)=\mathit{const}$, we expect that $\tilde q(c)\approx \mathit{const}$, as a consequence of the optimization process. However, in practice, we observe that classes are typically far from uniformly utilized. We suspect that our self-distillation learning framework leads to over-confident class predictions for a fraction of classes in the early training phase. 

To limit the effect of aligning the cluster predictions and take into account the individual cluster probability of the mini-batches $\tilde q_t^i(c)$, we introduce a hyperparameter $\beta \in (0.5, 1]$ in \Cref{eq:objective} to avoid collapsing all sample pairs in a single cluster. Without affecting the optimal solution, we rewrite \Cref{eq:objective} as

\begin{equation}
    {\widetilde{\pmi}\,}^i(x,x')= \log \sum_{c=1}^C \frac{
    \left( \pstud^i(c|x) \pteach^i(c|x')\right)^\beta}
    {\tilde q_t^i(c)},
\label{eq:objective-pmi}
\end{equation}

where $i \in \{ 1, \dots, H \}$ is the head index. 

\subsubsection{Intuition for $\beta$.} We now provide an explanation of how the above equation addresses the discussed challenges of image clustering. \Cref{eq:objective-pmi} consists of two parts inside the $\log$ sum: the numerator $\left( \pstud^i(c|x) \pteach^i(c|x')\right)^\beta$ encourages the class assignment of a positive pair to align (consistency) and is maximal when this assignment is one-hot. The denominator $\tilde q_t^i(c)$ promotes a uniform cluster distribution by up-weighing the summand corresponding to classes with low probability ($\tilde q_t^i(c)$ is low). In effect, $\beta$ balances these two effects by reducing the influence of the numerator, and thus degenerative solutions are avoided. 

Note that for $\beta=0.5$, the loss corresponds to the Bhattacharyya distance \cite{bhattacharyya1946measure} if $\tilde q_t^i(c)=const$. The Bhattacharyya distance can be minimal even if $q_t^i$ is far from one-hot. Moreover, if utilization of all classes is not required -- for example, as in the case of overclustering -- we set $\beta=1$. Crucially, $\beta$ is bounded, and we empirically found that the value of $0.6$ consistently avoids degenerative solutions across the majority of datasets, as opposed to existing clustering methods \cite{scan,sscn}. In addition, we propose an experimental strategy to choose $\beta$ without access to the ground-truth labels, as explained in \Cref{sec:experimental-results}. The symmetrized loss from \Cref{eq:objective-pmi} is defined as $\mathcal{L}^i(x,x')$ in analogy to \Cref{eq:symmetrized-loss}.

\subsection{Teacher-guided Instance Weighting}
\label{subsec:weight}
As discussed in \Cref{intro}, the mined $k$-NN in the feature space of $g(\cdot)$ tend to be noisy. For this reason, we introduce an instance weighting term for each head $i$ given by
\begin{equation}
    w_{i}(x,x') = \sum_{c=1}^C \pteach^i(c|x)\pteach^i(c|x').
\end{equation}
Intuitively, $w_{i}(x,x')$ acts as a guidance term that assigns a higher weight to true positive pairs compared to false positive ones. Importantly, $ w_{i}(x,x')$ relies only on the predictions of the teacher. The rationale behind this is that model averaging over training iterations tends to produce more accurate predictions \cite{tarvainen2017mean,polyak1992acceleration}. We call this setup teacher-weighted pointwise mutual information (WPMI). The final objective for each separate head $i$ is given by  $\mathcal{L}^{i}_{\operatorname{WPMI}}(x,x') \coloneqq   w_{i}(x,x') \mathcal{L}^{i}(x,x').$

\subsubsection{TEMI: Teacher Ensemble-weighted pointwise Mutual Information}
To compensate for the noisy NN pairs based on the feature space of $g$, we further propose to aggregate information from multiple heads, in contrast to previous works \cite{scan} that employ independent heads. For this purpose, we compute a scalar for each image pair using the mean weight across the heads, which is conceptually similar to model ensembling. We thus call this loss TEMI (teacher ensemble-weighted pointwise mutual information) defined by
\begin{equation}
     \mathcal{L}^{i}_{\operatorname{TEMI}}(x,x') \coloneqq   \frac{1}{H} \sum_{j=1}^H w_{j}(x,x') \mathcal{L}^{i}(x,x').
\end{equation}

\section{Experimental evaluation}
\label{experimental}

\subsection{Datasets, Metrics and Implementation Details}
The proposed method (TEMI) is evaluated on five common benchmark datasets, namely CIFAR10, CIFAR20, CIFAR100 \cite{cifar}, STL10 \cite{stl10}, and ImageNet \cite{deng2009imagenet}. CIFAR10, CIFAR20 and CIFAR100 contain $50K$ training images, STL10 contains $5K$ training samples, and ImageNet has \num{1281167} training samples. CIFAR20 has the same training data as CIFAR100 with $20$ superclasses derived from the ground-truth labels. We resize all images to $224\times224$. The training set is used during the optimization phase, while the evaluations are carried out on the validation set. Additional information can be found in the supplementary material.

To quantify the clustering performance, we report the clustering accuracy (ACC) and the adjusted random index (ARI). To estimate the accuracy, the one-to-one mapping between cluster predictions and ground-truth labels is computed by the Hungarian matching algorithm \cite{kuhn1955hungarian}. For our overclustering experiments, we report the adjusted mutual information (AMI). Finally, we establish two baselines: a) $k$-means and b) the SCAN clustering loss within our self-distillation framework. For a fair comparison with existing methods, we assume to know in advance the number of ground-truth labels. Concerning the hyperparameters, we set $H=50$ and $\beta=0.6$ for clustering, while we set $\beta=1$ for overclustering. We use $25$-NN on ImageNet and $50$-NN for the remaining datasets. We used the AdamW optimizer \cite{adamw} for $200$ epochs with a batch size of $512$, with a learning rate of $10^{-4}$, weight decay of $10^{-4}$ and report results at the end of training. Unlike previous methods \cite{scan}, we found that augmentations (RandAugment \cite{cubuk2020randaugment}, and the ones from \cite{simclr}) were not improving the clustering metrics when training with $k$-NN pairs. Hence, we precomputed the feature representations, which enables training the clustering heads on a single GPU with $12$GB of VRAM within $24$ hours.

\input{figures/merged_figs}

\subsection{Experimental Results}
\label{sec:experimental-results}
We first present a strategy to choose $\beta \in (0.5,1]$. As depicted in \Cref{fig:beta_ent}, an accurate model, $p_t(c|x)$ should be able to maintain a high entropy $\Entropy(\pteach(c))$, while maintaining its discriminative power. To quantify the latter, we use the conditional entropy $\Entropy(\pteach(c|x))$. The lower the value of $\Entropy(\pteach(c|x))$, the more discriminative the predictions. The extreme case $\Entropy(\pteach(c|x))=0$ corresponds to a one-hot distribution. Thus, we propose to pick the lowest value of $\beta$ such that $\Entropy(\pteach(c|x))$ remains sufficiently low. We experimentally found $0.6$ to work consistently well across models and datasets.

As shown in \Cref{table:losses}, an average accuracy gain of $5.0$\% over $k$-means is found for CIFAR100 and ImageNet, even with the plain PMI setup. Introducing multiple heads in PMI further improves the obtained results by an average gain of $0.8$\%. Critically, for our best setup (TEMI), we observe an average gain of $8.1$\% and $3.7$\% compared to $k$-means and the SCAN clustering loss, respectively. Note that even $16$ heads were sufficient to get similar performance, specifically less
than 1\% accuracy deterioration compared to 50 heads.

\input{tables/loss_gains.tex}

To study the applicability of our method, we then applied our best setup (TEMI) to various publicly available pretrained models, as shown in Fig \ref{fig:main}. Therein, we report an average accuracy gain of $6.1$\% and $12.1$\% compared to $k$-means on ImageNet and CIFAR100 across $17$ different pretrained models. More specifically, TEMI MSN ViT-L/14 and TEMI DINO ViT-B/16 are the best-performing self-supervised methods on ImageNet ($61.6$\% ACC) and CIFAR100 ($67.1$\% ACC). Moreover, CLIP-based backbones have the highest ACC increase over $k$-means when trained with TEMI, precisely $10.7$\% on ImageNet and $14.1$\% on CIFAR100.

\input{tables/merge_tables_cluster.tex}
Concerning the supervised pretrained models in Fig. \ref{fig:main}, we demonstrate that ConvNext-L outperforms ViT-L on ImageNet, precisely by $2.7$\% on ACC with TEMI. However, the supervised ViT-L surpasses ConvNext-L by a large margin of $22.3$\% in ACC when benchmarked on CIFAR100 with TEMI. Among the architectures investigated, large ViTs learn the most transferable label-related features, even without supervised fine-tuning. This finding is consistent with \cite{naseer2021intriguing}.

\textbf{Clustering and overclustering results on ImageNet.} Regarding ImageNet, we compare various self-supervised architectures that were trained without any external data, as depicted in \Cref{tab:in1k}. Using the same architecture (Resnet50) as current state-of-the-art models (SSCN \cite{sscn}), TEMI achieves an improvement of $4.1$\% in ACC. With MSN ViT-L/16 as the backbone, we push the state-of-the-art ACC on ImageNet to $61.6$\%, resulting in a substantial gain of $20.5$\% compared to SSCN. The obtained results strongly indicate that first learning the augmentation-invariant features and then focusing on learning the invariances w.r.t.\ images that belong to the same class is an effective strategy for image clustering. Incentivized by the above observation, we investigate the overclustering performance in \Cref{tab:result_overclustering_in1k} by adopting the setup from \cite{pcl}. More concretely, we use $25$K clusters and set $\beta=1$ as in ProPos. We almost match the performance of ProPos \cite{propos} with TEMI DINO Resnet50 without tuning the number of clusters or any other hyperparameter while reaching a considerable gain of $7.4$\% in AMI with TEMI DINO ViT-B/16.

\textbf{Small-scale benchmarks.} In \Cref{tabel:small-datasets-sota}, the transfer performance on three small-scale datasets is evaluated. TEMI DINO ViT-B backbone has inarguably the best transfer performance, outperforming the ACC of ProPos by 4.6\% and TSP by 2.9\% on average. It is worth pointing out that TSP \cite{tsp} uses the same pretrained model, and it is thus a fair comparison. Ultimately, we notice a large accuracy gap between clustering methods and probing in CIFAR20, which suggests that the superclass structure cannot be inferred from the visual input. For instance, clocks, lamps, and telephones are grouped into household electrical devices.

\input{tables/merge_last_tables}

\textbf{Analysis on noise (false positives) from the NN of $g$.} As shown in Tab.\ 5, when keeping only the true positive neighbours from 50-NN, we increased the performance from $67.1$ $\rightarrow$ $82.6$ using TEMI DINO ViT-B on CIFAR100. We also show that the head weighting term of TEMI in Eq.\ 10 is not needed, highlighting that TEMI is designed for the noisy pairs obtained from $k$-NN. As a reference, DINO ViT-Base has 72\% true positive pairs in 20-NN and 66\% in the mined 50-NN on CIFAR100.

\subsection{Discussion}
\label{discussion}
\noindent\textbf{How expressive can an image classifier be by only training with NN pairs?} We examine the training accuracy in \Cref{table:true-pairs}, by training with the true positive pairs from the computed $k$-NN. The 98.6\% training accuracy on CIFAR100 with TEMI DINO ViT-B/16 indicates that it is possible to train a powerful unsupervised image classifier by only relying on pairs. In fact, we observe that we almost match the supervised linear probing accuracy on CIFAR100 (84.1\% vs 85.3\%). Still, we identify cases where the human-annotated label is ambiguous and cannot be determined solely by the visual signal (supplementary material).

\noindent\textbf{What is the impact of the instance weighting term?} 
After training, we examined the actual value of the instance weighting term $w(x,x')$. To this end, we computed the mean weights for true positives and false positives sampled from $50$-NN within the CIFAR100 validation set, which take the values $0.76$ and $0.40$, respectively. Furthermore, $w(x,x')$ has a negative impact when only true positive pairs are considered during training (\Cref{table:true-pairs}). This is an expected behavior, as a fraction of true positive pairs will be down-weighted by $w(x,x')$ due to low feature similarity (i.e. digital and analog clocks).

\noindent\textbf{How discriminative are the cluster assignments of TEMI?}
Besides Fig.\ 2, we quantify the discriminative power of TEMI by computing the mean and median maximum softmax probability (MSP \cite{hendrycks2016msp}). We calculate a mean and median MSP of 88.5\% and 98.9\% on CIFAR100 and 85.3\% and 99.2\% on ImageNet. The computed results verify that the introduced framework results in discriminative predictions.

\noindent{\textbf{Joint learning of encoder and cluster head.}} When jointly training the pretrained backbone with the already trained head, we observed a performance increase only when the pretraining dataset was different from the downstream dataset (67.1\% $\rightarrow$ 70.9\% ACC on CIFAR100 using the ImageNet-pretrained DINO ViT-B/16). We hypothesize this enables learning features specific to the training distribution. Still, the structure in the latent space of the pretrained model is required for TEMI to determine the $k$-NN.

\section{Conclusion}
\label{conclusion}
In this paper, a novel and general self-distillation framework for image clustering was proposed that can achieve competitive results almost out of the box. In addition, a new objective based on pointwise mutual information was presented. After studying the performance of $17$ pretrained models, it was shown that TEMI can be used with any pretraining with significant improvements over $k$-means. Finally, new state-of-the-art results were achieved on ImageNet both for clustering and overclustering, leveraging self-supervised ViTs. To conclude, future works are encouraged to explore the connection between image clustering and representation learning in greater depth.
\bibliography{bibl}

\include{appendix}
\end{document}

%% file: figures/scatter_plot_kmeans.tex
\begin{figure*}[t]
\begin{center}
\includegraphics[width=1.0\textwidth]{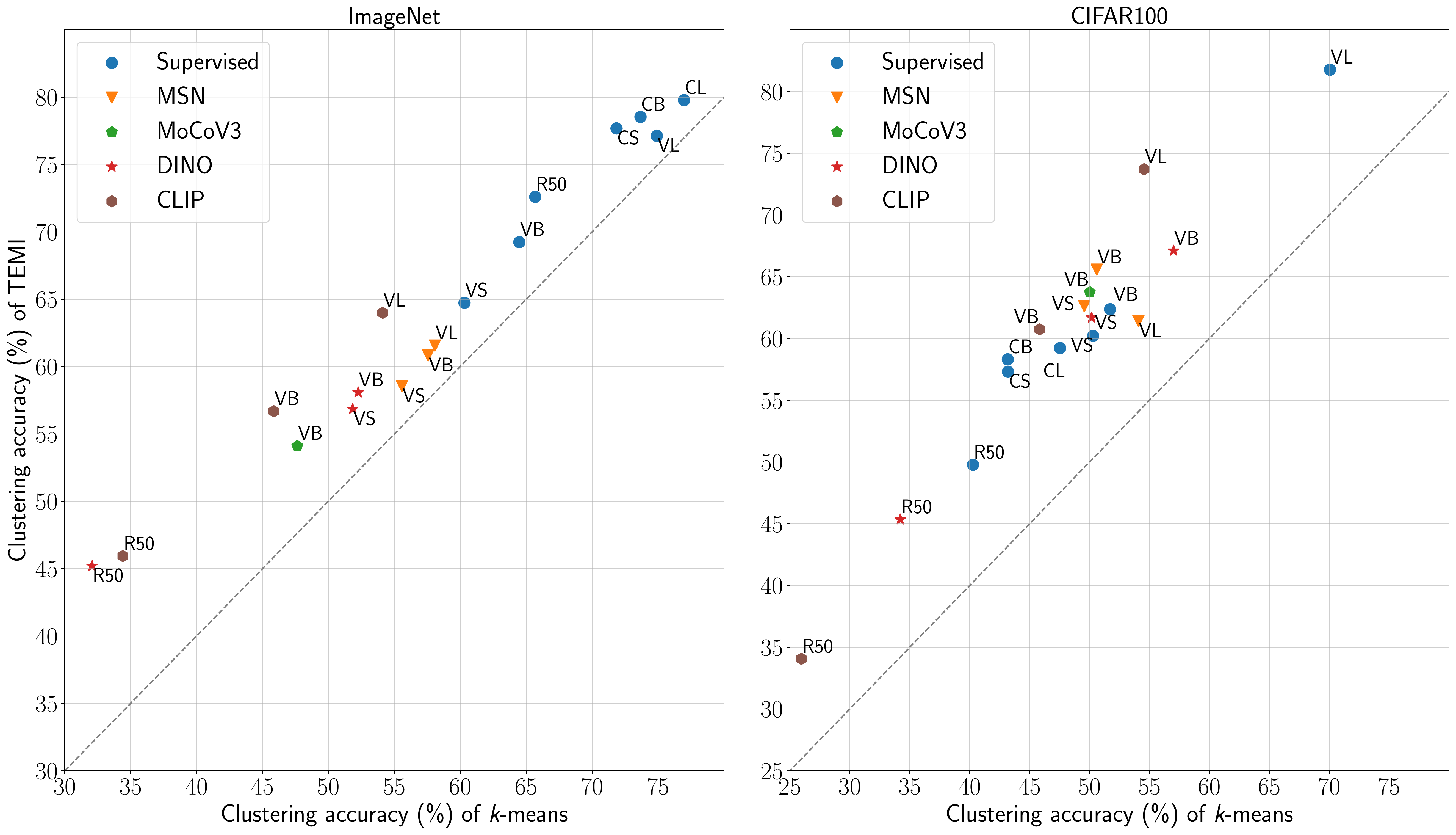}
\end{center}
\caption{
\textbf{Clustering accuracies on ImageNet (left) and CIFAR100 (right) across 17 pretrained models.} Supervised and self-supervised models (MSN, MoCoV3, DINO) were pretrained on ImageNet. R50 stands for ResNet50 \cite{resnet}, C for ConvNext \cite{convnext}, and V for Vision Transformer \cite{vit}. Small (S), Base (B), and Large (L) indicate the size of the models. The vertical distance of each data point to the diagonal (dashed line) shows the improvement of our method (TEMI) over $k$-means. Best viewed in color.} 
\label{fig:main}
\end{figure*}

%% file: figures/merged_figs.tex
\begin{figure}[htbp]
\begin{minipage}{.5\textwidth}
     \includegraphics[width=\textwidth]{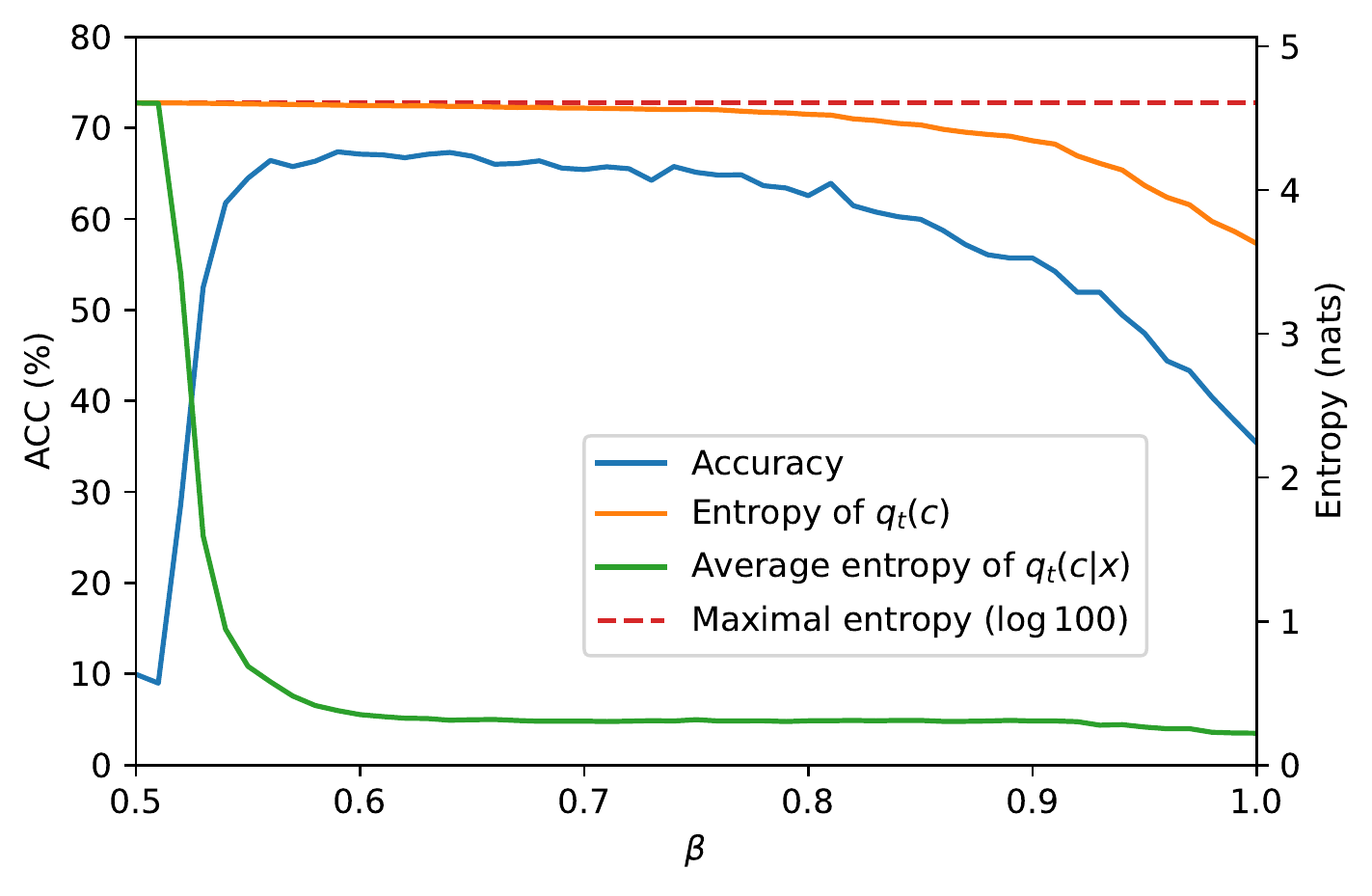}
\caption{\textbf{Effect of $\beta$ on the validation accuracy and on the entropy of $\pteach(c|x)$ and $\pteach(c)$ on CIFAR100}. The values are computed using TEMI DINO ViT-B/16. The dashed horizontal line illustrates the maximal possible entropy, i.e. $\log C$. A high entropy of $\pteach(c)$ indicates that the clusters are almost uniformly utilized, while a low entropy of $\pteach(c|x)$ indicates highly confident predictions (one-hot).}
\label{fig:beta_ent}
\end{minipage}
\hfill
\begin{minipage}{.45\textwidth}
\begin{center}
\includegraphics[width=\textwidth]{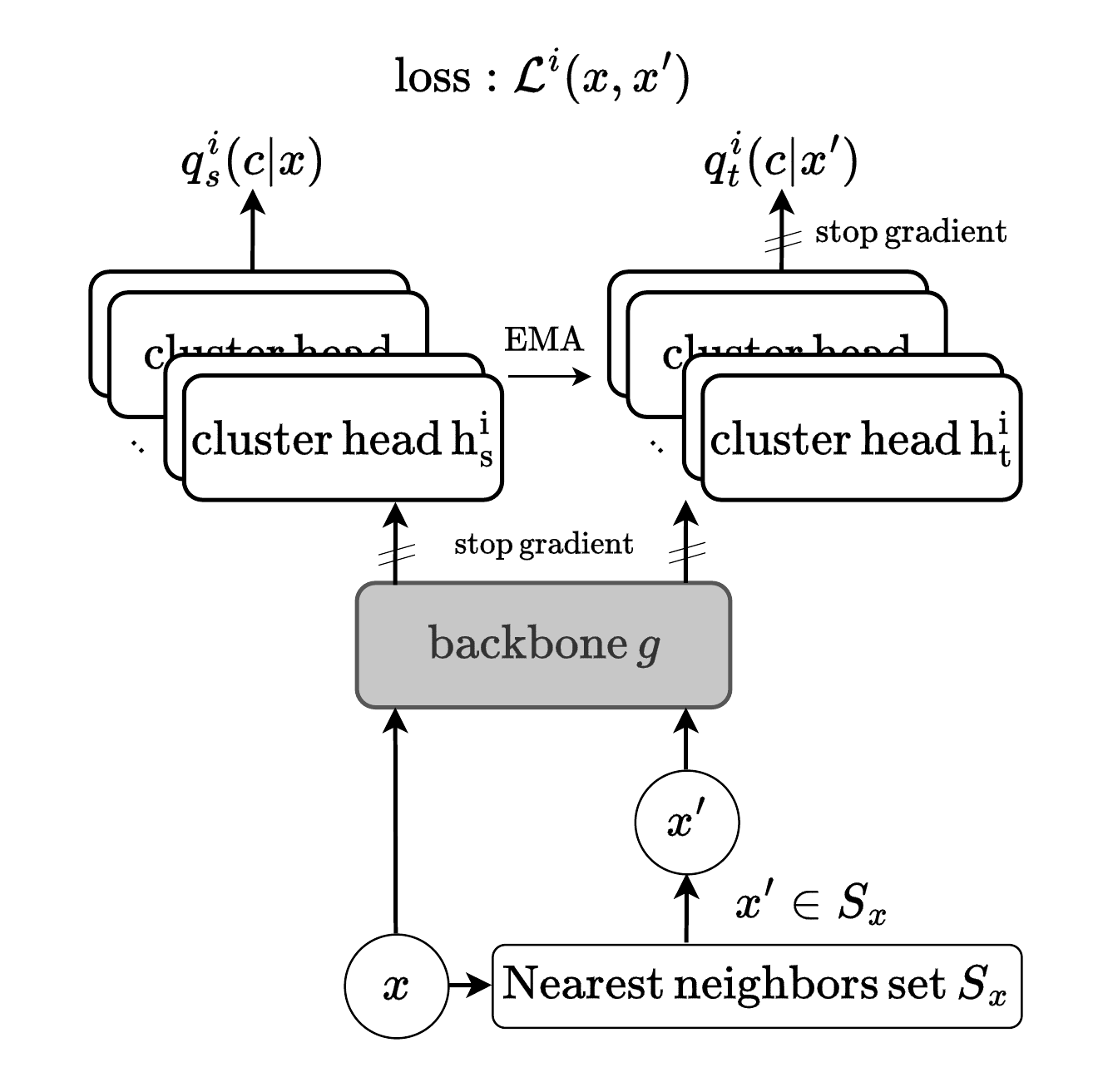}
\end{center}
\caption{\textbf{An overview of the proposed self-distillation clustering framework.} The nearest neighbors are mined in the feature space of $g$. EMA refers to the exponential moving average over the parameters.}
\label{fig:diagram}
\end{minipage}
\end{figure}

%% file: tables/loss_gains.tex
\begin{wraptable}{O}{.55\textwidth}

 \begin{center}
  \begin{tabular}{lccc}
    \toprule
    Method  & Heads & CIFAR100 & ImageNet  \\
    \midrule
    \textit{k}-means & - & 56.99 & 52.26  \\
    SCAN & 50 & 62.6$\pm$0.94 & 55.6$\pm$0.15 \\
    \hline
    PMI & 1 & 61.6$\pm$0.41 & 57.5$\pm$0.22  \\
    WMI & 1 & 63.4$\pm$1.89 & 56.5$\pm$0.41 \\
    \hline
    PMI & 50 & 63.1$\pm$0.56 & 57.7$\pm$0.06  \\
    WPMI & 50 & 65.6$\pm$1.04 & 57.0$\pm$0.38 \\
    TEMI & 50 & \textbf{67.1$\pm$1.30} & \textbf{58.4$\pm$0.22}  \\
    
    \bottomrule
  \end{tabular}
  \end{center}
    \caption{\textbf{Ablation study for the TEMI objective}. All the experiments were conducted with $\beta=0.6$, and DINO ViT-B/16 as the backbone model. The clustering accuracy is reported in \%.}
    \label{table:losses}
\end{wraptable}

%% file: tables/merge_tables_cluster.tex
\begin{table}[htbp]
\begin{minipage}{.45\textwidth}
  \begin{tabular}{lccc}
    \toprule
    Method  & Arch.   & ACC &   ARI \\
    \hline
    SeLa \cite{sela} & Resnet50  & 30.5 &  16.2  \\
    SCAN \cite{scan} & Resnet50    & 39.9 & 27.5  \\
    SSCN \cite{sscn} & Resnet50    & 41.1  & 29.5 \\
    \midrule
     \multicolumn{4}{l}{\textit{Our method}}\\	
    TEMI DINO  & Resnet50    & 45.2  & 31.3 \\
    TEMI DINO & ViT-B/16     & 58.4  &  45.9 \\
    TEMI MSN   & ViT-L/16  & \textbf{61.6}  &  \textbf{48.4} \\
    \bottomrule
  \end{tabular}
  \newline \newline
\caption{\textbf{Clustering results for the ImageNet validation set, without using additional data.} Evaluation metrics include clustering accuracy (ACC), and adjusted random index (ARI) in \%. All our models are pretrained on ImageNet.}
\label{tab:in1k}
\end{minipage}
\hfill
\begin{minipage}{.45\textwidth}

\begin{tabular}{lc}
  \toprule
  Method      & AMI (\%)   \\
  \midrule
  DeepCluster~\cite{deepcluster}  & 28.1 \\
  MoCo~\cite{moco}        & 28.5 \\
  PCL~\cite{pcl}     & 41.0 \\
  ProPos~\cite{propos}     & 52.5 \\
  \midrule
  TEMI DINO Resnet50 & 51.8$\pm$0.1 \\
  TEMI DINO ViT-B/16 & \textbf{59.9$\pm$0.2} \\
  TEMI MSN ViT-L/16 & 58.8$\pm$0.5 \\
  \bottomrule
  \end{tabular}
  \newline \newline
  \caption{\textbf{Overclustering results on the ImageNet validation set.} The adjusted mutual information (AMI) score for 25K clusters is reported, as in \cite{pcl}. For all experiments, we set $\beta=1$.}
  \label{tab:result_overclustering_in1k}
\end{minipage}
\end{table}

%% file: tables/merge_last_tables.tex
\begin{table}[htbp]
\begin{minipage}{.5\textwidth}

\begin{adjustbox}{width=\columnwidth}
\begin{tabular}{ l c c c c   }
			\toprule
			{Methods} & {CIFAR10}     &  {CIFAR20}      & {STL10} \\\hline
NNM \cite{dang2021nearest} & 84.3 & 47.7  & 80.8 \\  
PCL \cite{pcl} &  87.4 & 52.6  & 41.0 \\ %
SCAN \cite{scan} &  88.3  & 50.7 &  80.9  \\
SPICE \cite{niu2022spice} &  92.6  & 53.8  & 93.8 \\
ProPos$^\star$ \cite{propos} & 94.3 & 61.4 & 86.7 \\ %
TSP$^{\dagger}$ \cite{tsp}  & 94.0 & 55.6 & 97.9  \\ %
\hline
TEMI$^{\dagger}$ &  \textbf{94.5} &  \textbf{63.2}  &  \textbf{98.5} &  \\ 
\hline 
  \multicolumn{4}{l}{\textit{supervised baseline}}\\
  \multicolumn{1}{l}{Probing$^{\dagger}$ } & 96.8  & 89.5 & 99.2 \\ 
			\bottomrule
		\end{tabular}
\end{adjustbox}
\newline \newline
\caption{\textbf{Clustering accuracy on small datasets.} Methods with $^{\dagger}$ use DINO ViT-B pretrained on ImageNet, while $^{\star}$ indicates methods that include the validation split during training.}
\label{tabel:small-datasets-sota}

\end{minipage}
\hfill
\begin{minipage}{.45\textwidth}

\begin{center}
\begin{tabular}{lcc}
\toprule
Loss &  Val.  & Train  \\
 &   ACC &  ACC \\
\midrule
PMI   & \textbf{84.1$\pm$0.36}  & \textbf{98.6$\pm$0.38} \\
TEMI & 82.6$\pm$0.67  & 96.5$\pm$0.88 \\
\hline
 \multicolumn{3}{l}{\textit{supervised baseline}}\\
Probing & 85.3  & 99.3 \\
\bottomrule
\end{tabular} 
\end{center}

 \caption{\textbf{Clustering accuracies on CIFAR100 when training only with the true positive NN pairs using TEMI DINO ViT-B/16.}}
 \label{table:true-pairs}

\end{minipage}
\end{table}

%% file: appendix.tex
\appendix

\section{Further discussion points}

\subsection{How to choose $\beta$ for a new dataset?}
Here, we provide a more detailed explanation of Fig. 3 (in the main text) on how to pick $\beta \in (0.5, 1]$ without access to ground-truth data. First, the motivation behind $\beta$ is to avoid the imbalanced growth of clusters during training. The closer $\beta$ is to $0.5$, the more balanced the clusters (clusters contain a similar number of examples). The reason is that the loss contribution to assign each training sample a single class is reduced for smaller $\beta$. However, for $\beta=0.5$, each sample occupies all clusters with equal probability. Consequently, we have to impose $\beta>0.5$, but $\beta$ should be sufficiently close to $0.5$. We take for $\beta$ the value when the conditional entropy, $E_x[\sum_c -q(c|x) log q(c|x)]$ (Fig. 3 in the main text, green line), is starting to become constantly low. We experimentally found $0.6$ to work consistently well across models and datasets. An exception is CIFAR20, where we used $\beta=0.55$ since superclasses are conceptually a form of under-clustering.

\subsection{Fine-tuning the pretrained backbone with TEMI}
Given a pretrained backbone network, fine-tuning the backbone simultaneously with training randomly initialized heads gave bad results. However, fine-tuning the backbone simultaneously with fine-tuning the already trained head with TEMI yielded superior performance but only when the pretraining dataset was different from the downstream dataset, e.g. $67.1 \rightarrow 70.9$ for CIFAR100 using DINO ViT-B/16 pretrained on ImageNet as the backbone model.

\subsection{Additional computational complexity from multiple heads}
In theory, the computational time complexity of TEMI by adding multiple heads is linear, given a sequential implementation. In practice, due to GPU-related optimizations, it's much faster. In fact, training on a single Nvidia A100 GPU takes only 4 GB of memory with $50$ heads on CIFAR100. Training takes just about 45 minutes because we precompute the feature representations while training with just one head takes about 5 minutes.

\input{figures/pk_histogram}

\subsection{How discriminative are the resulted cluster assignments of TEMI?}
Besides Fig. 2 in the main text, we quantify the discriminative power of TEMI by computing the mean and median maximum softmax probability (MSP \cite{hendrycks2016msp}). We calculate a mean and median MSP of 88.5\% and 98.9\% on CIFAR100 and 85.3\% and 99.2\% on ImageNet. The computed results verify that the introduced framework results in discriminative predictions.

\subsection{Are multiple heads necessary?}
The idea of using multiple heads is inspired by previous works, such as SCAN \cite{scan} and SSCN \cite{sscn}. The proposed PMI objective does not require multiple heads by design. As shown in Table 3, we experimentally observed an initial gain of 0.8\% by adding independent heads (PMI and WMI setup with 50 heads). Importantly, one of our core novelties lies in the combination of the teacher predictions from multiple heads, Eq. (10), in the main text, which provides superior results compared to having independent heads (Table 3 in the main text).  Overall, we find the reported performances saturate quickly with more heads and are already close to the maximum for 16 heads on CIFAR100. Based on our first results on CIFAR100, we fixed the number of heads to 50 for all models and datasets.

\subsection{Contrastive versus non-contrastive self-supervised pretraining for image clustering.}
The performance gap between contrastive (MoCoV3 ViT-B) and non-contrastive (DINO ViT-B) backbones likely originates from the homogeneous distribution of examples in feature space as part of the contrastive learning objective, which likely attenuates the necessary structure in feature space for image clustering \cite{wang2020understanding,propos}.

\input{tables/sota_small_datasets} 
\input{tables/cifar100_only}

\subsection{A Note on CIFAR100 VS CIFAR20}
We observe that previous works have established the CIFAR20 as a clustering benchmark. However, we believe that the CIFAR20 superclasses are not an ideal benchmark for image clustering. In the reported results in \cref{tabel:small-datasets-sota}, one can easily notice that all models perform worse in CIFAR20 than in CIFAR100 (\cref{tabel:cifar100-only}). Examples that justify the performance gap include a) clocks, computer keyboards, lamps, telephones, and televisions are grouped into household electrical devices, b) bridges, castles, houses, roads, and skyscrapers are grouped into large man-made outdoor things, and c) bears, leopards, lions, and wolfs are grouped into carnivores. These examples illustrate that the superclasses are not separable from the pixel information alone. To this end, we would like to encourage future works to adopt CIFAR100 as a benchmark for image clustering, as shown in \cref{tabel:cifar100-only}.

\input{tables/sota_in_subs-comparison}

\clearpage
\section{Proof of Theorem 1}
\label{proofs}

\input{proofs}

\section{Additional implementation details.}
To enforce reproducibility, the means and standard deviations are reported for all our experiments and metrics, computed over $3$ independent runs with different seeds. For a fair comparison with SCAN, we tune its entropy regularization hyperparameter, $\lambda$, based on a grid search and use the value $\lambda = 4$. Crucially, we found that some pretrained models (i.e.\ MSN) produce unnormalized features. For that reason, we standardize the features of all models before feeding them to the clustering heads. For the linear probing experiments, we trained a linear layer using the Adam \cite{adam} optimizer with a learning rate of $10^{-3}$ and weight decay of $10^{-3}$.

\section{Randomly sampled images from TEMI cluster assignments}
\enlargethispage{1\baselineskip}
\label{appendix:random-image-samples}
\begin{figure}[H]
\onecolumn\includegraphics[width=\textwidth]{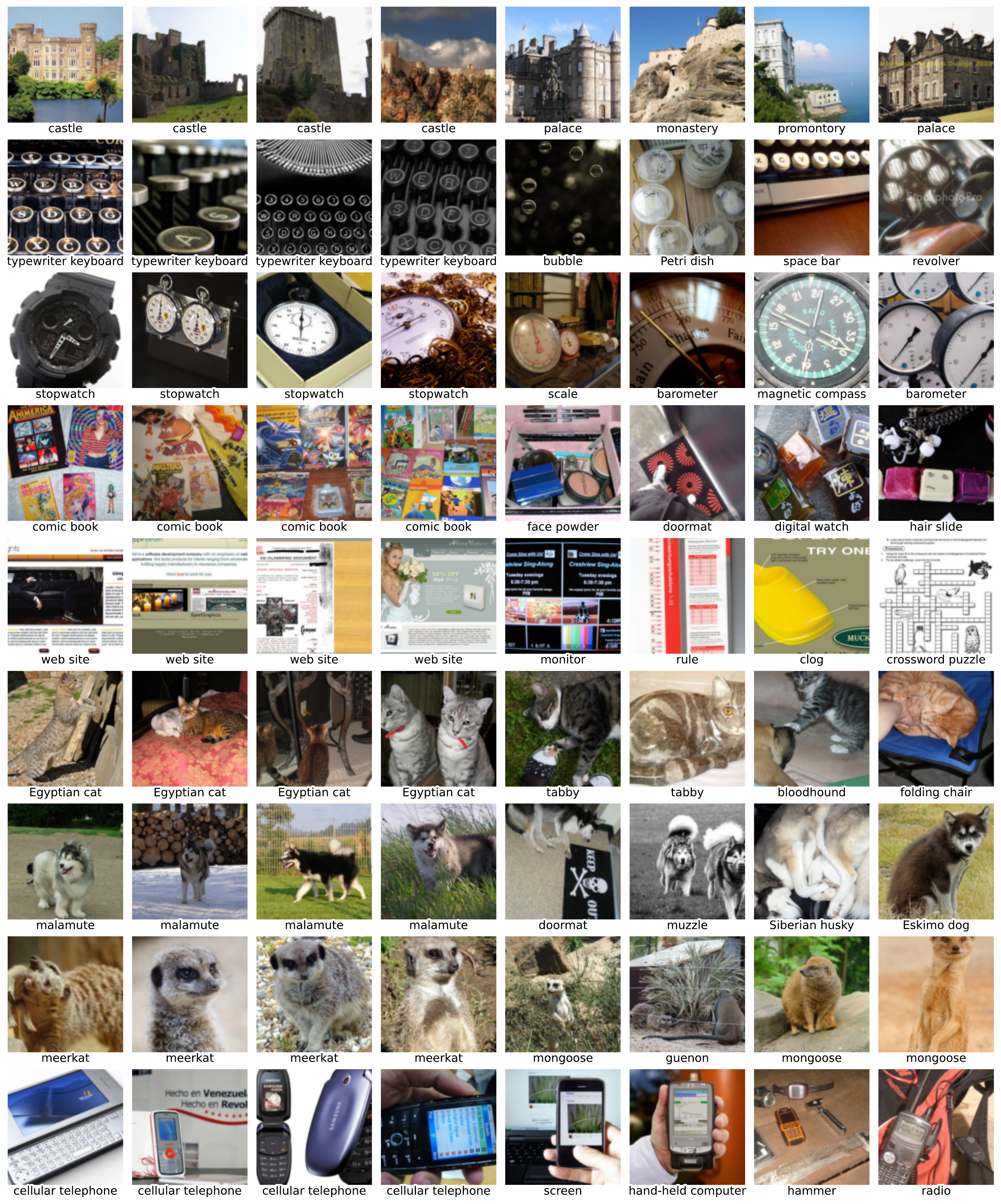}
\vspace{0.3cm}
\caption{\textbf{Randomly sampled images from the ImageNet dataset that are assigned in the same cluster using the TEMI MSN ViT-L/16 model}. The ground-truth label is indicated in the text under the image. The images in each row are assigned to the same cluster. The first four columns correspond to correctly classified images, while the last four are examples of misclassified images.}
\end{figure}
\clearpage
\begin{figure*}[h!]
\includegraphics[width=\textwidth]{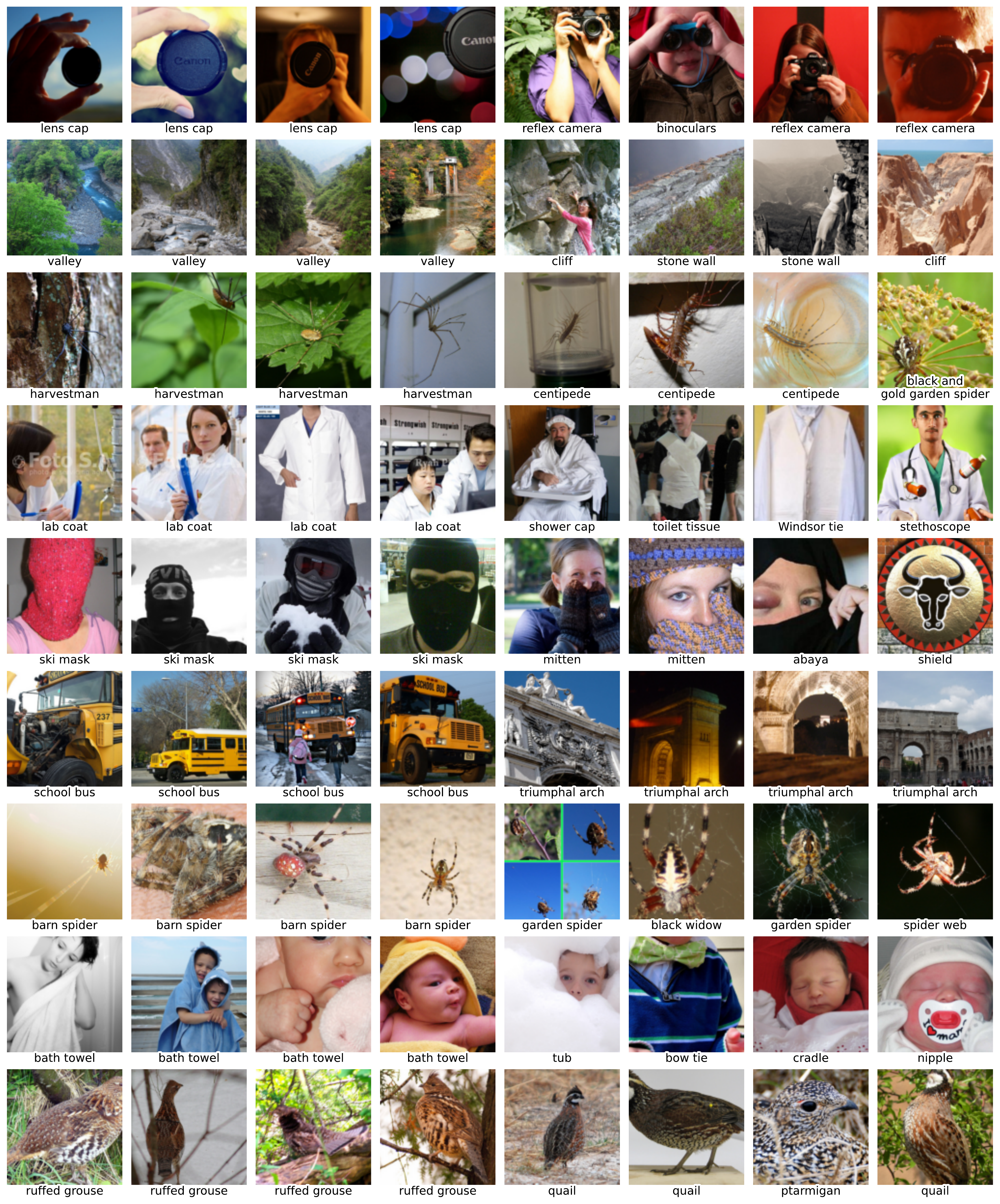}
\vspace{0.3cm}
\caption{\textbf{More randomly sampled images from the ImageNet dataset that are assigned in the same cluster.}}
\end{figure*}

\input{tables/sup-data-info}
\input{tables/hp_pretrain.tex}
\input{tables/hp_finetune.tex}

\input{tables/sup-additional-results-sota}

%% file: figures/pk_histogram.tex
\begin{figure*}[h]
    \begin{center}
    \subfigure[CIFAR100 using TEMI DINO ViT-B/16]{
    \includegraphics[width=0.47\columnwidth]{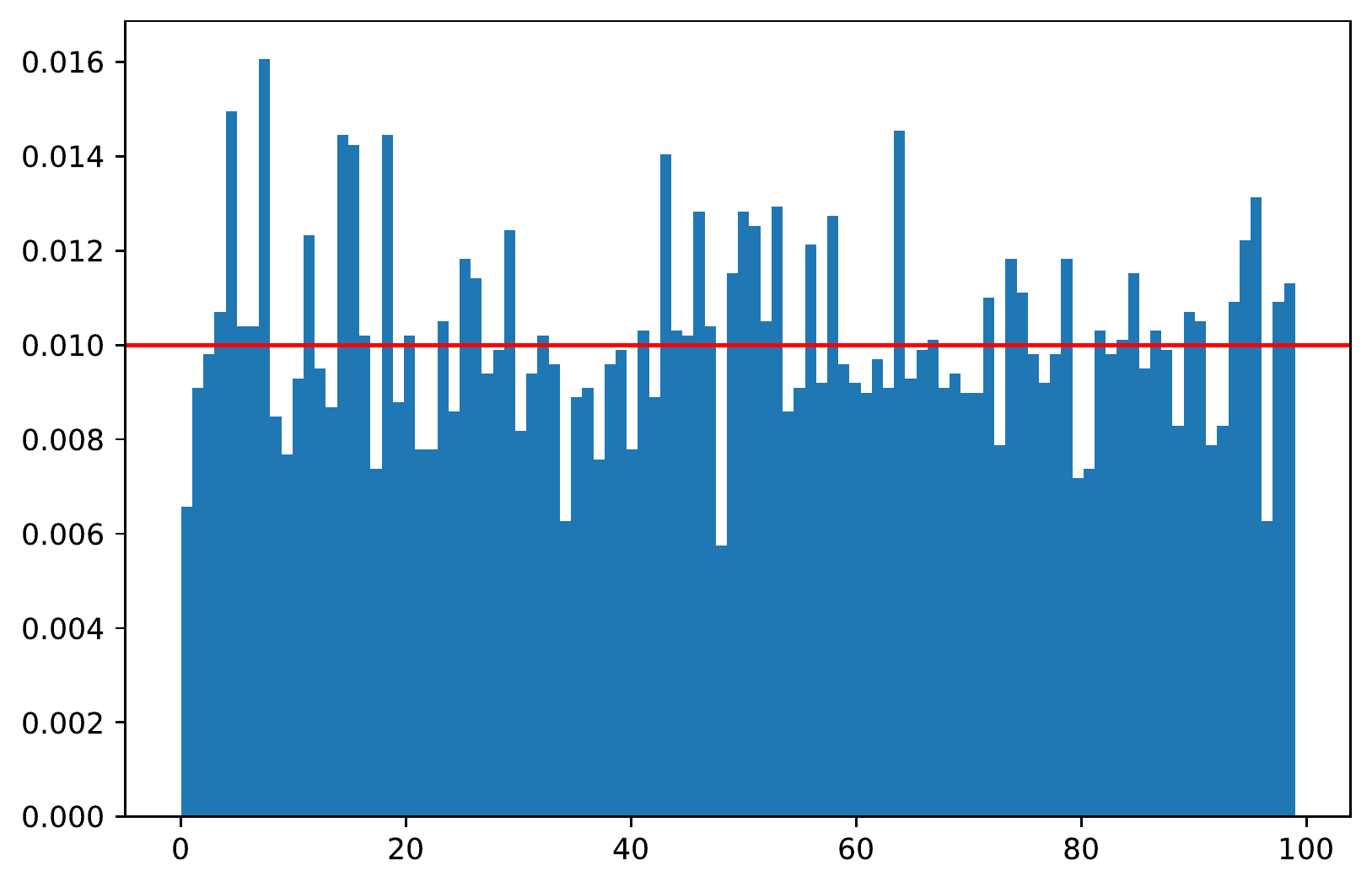}
    }%
    \subfigure[ImageNet using TEMI MSN ViT-L/16]{
    \includegraphics[width=0.47\columnwidth]{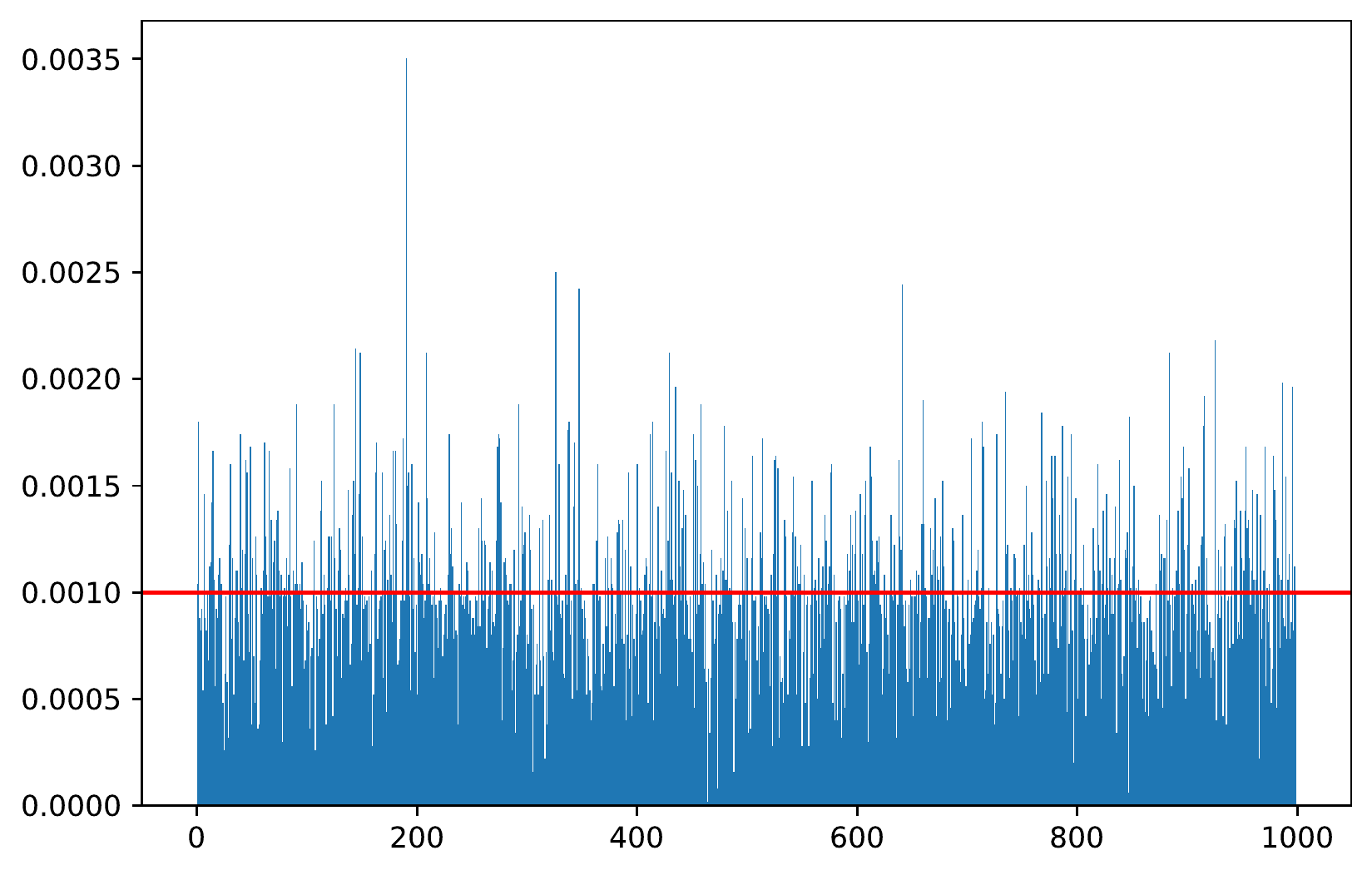}
    }
    \newline
    \caption{\textbf{Histogram of cluster assignments on different datasets}. The horizontal red line illustrates the ideal histogram, where all clusters would be uniformly utilized. We also compute the KL divergence between the predictions and the uniform distribution on CIFAR100 and ImageNet, which is $1.5\cdot10^{-2}$ and $5\cdot10^{-2}$, respectively. The predictions would be uniform in the extreme case where the KL divergence is 0.}
    \label{fig:pk_histogram}
     \end{center}
\end{figure*}

%% file: tables/sota_small_datasets.tex
\begin{table*}
	
 \begin{center}
 \begin{adjustbox}{width=1\columnwidth}
		\begin{tabular}{ l c c c c c c c c c c c }
			\toprule
			\multicolumn{1}{c}{Datasets} & \multicolumn{3}{c}{CIFAR10} &      & \multicolumn{3}{c}{CIFAR20} &      & \multicolumn{3}{c}{STL10} \\
			\cmidrule{2-4}\cmidrule{6-8}\cmidrule{10-12}    \multicolumn{1}{c}{Methods} & NMI(\%)  & ACC(\%) & ARI(\%) &      & NMI(\%) & ACC(\%) & ARI(\%) &      & NMI(\%) &ACC(\%)& ARI(\%)\\
\hline
DAC (\citeauthor{chang2017deep}) & 39.6 & 52.2 & 30.6 &      & 18.5 & 23.8 & 8.8  &      & 36.6 & 47   & 25.7 \\ 
DCCM (\citeauthor{wu2019deep}) & 49.6 & 62.3 & 40.8 &      & 28.5 & 32.7 & 17.3 &      & 37.6 & 48.2 & 26.2 \\
PICA (\citeauthor{huang2020deep}) & 59.1 & 69.6 & 51.2 &      & 31   & 33.7 & 17.1 &      & 61.1 & 71.3 & 53.1 \\  
NNM (\citeauthor{dang2021nearest}) & 74.8 & 84.3 & 70.9 &      & 48.4 & 47.7 & 31.6 &      & 69.4 & 80.8 & 65 \\  
PCL (\citeauthor{pcl}) &  80.2 & 87.4 & 76.6        &      & 52.8 & 52.6 & 36.3       &  & 71.8 & 41.0 & 67.0  \\ %
SCAN (\citeauthor{scan}) & 79.7 &  88.3 & 77.2 &      & 48.6 & 50.7 & 33.3 &      & 69.8 & 80.9 & 64.6 \\
SPICE (\citeauthor{niu2022spice}) & 86.5 &  92.6 & 85.2 &      & 56.7 & 53.8 & 38.7 &      & 87.2 & 93.8 & 87.0 \\
ProPos$^\star$ (\citeauthor{propos})  &  88.6 & 94.3 & 88.4   &      & 60.6 & 61.4 & 45.1   &    & 75.8 & 86.7  & 73.7  \\ %
TSP$^{\dagger}$ (\citeauthor{tsp})  & 88.0 & 94.0 & 87.5 &      &61.4 & 55.6 & 43.3 &      & 95.8 & 97.9 & 95.6 \\ %
\hline

	 TEMI DINO ViT-B/16$^{\dagger}$ & \textbf{88.6$\pm$0.05} &  \textbf{94.5$\pm$0.03} &  \textbf{88.5$\pm$0.08} &      & \textbf{65.4$\pm$0.45} &  \textbf{63.2$\pm$0.38} &  \textbf{48.9$\pm$0.21} &      & \textbf{96.5$\pm$0.13} &  \textbf{98.5$\pm$0.04} &  \textbf{96.8$\pm$0.09} \\

     TEMI MSN ViT-L/16$^{\dagger}$ & 82.9$\pm$0.16 &  90.0$\pm$0.14 &  80.7$\pm$0.22 &      & 59.8$\pm$0.04 &  57.8$\pm$0.42 &  42.5$\pm$0.08 &      & 93.6$\pm$1.10 &  96.7$\pm$0.89 &  93.0$\pm$1.74 \\
     
\hline 
   \multicolumn{12}{l}{\textit{(natural language) supervised pretraining}}\\

TEMI CLIP ViT-L/14$^{\dagger}$ &  92.6$\pm$0.13 &  96.9$\pm$0.07 &  93.2$\pm$0.15 &      &  64.5$\pm$0.12 &  61.8$\pm$1.47 &  46.8$\pm$1.17 &      & 96.4$\pm$0.79 &  97.4$\pm$0.69 &  94.9$\pm$1.26 \\

TEMI Sup. ViT-L/16$^{\dagger}$ & 91.8$\pm$0.65 &  96.0$\pm$0.53 &  91.6$\pm$1.02 &      & 65.0$\pm$0.89 &  58.4$\pm$0.98 &  45.4$\pm$1.41 &      & 82.7$\pm$2.94 &  84.6$\pm$2.37 &  73.9$\pm$2.77 \\
		\hline \hline
   \multicolumn{12}{l}{\textit{supervised baselines}}\\
  
  \multicolumn{1}{l}{Probing DINO ViT-B/16$^{\dagger}$ }   & 92.5 & 96.8 & 93.1 &     & 82.4 & 89.5 & 79.5 &      & 97.8 & 99.2 & 98.2 \\ 

       \multicolumn{1}{l}{Probing MSN ViT-L/16$^{\dagger}$} & 91.5 & 96.4 & 92.3 &   & 80.7    & 88.2 & 77.0 & & 96.8 &   98.8   & 97.4 \\

     \multicolumn{1}{l}{Probing CLIP ViT-L/14$^{\dagger}$}   & 95.1 & 98.1 & 95.8 &     & 85.7 & 91.7 & 83.6 &      & 99.2 & 99.7 & 99.4 \\ 
     
  \multicolumn{1}{l}{Probing Sup. ViT-L/16$^{\dagger}$}   & 91.5 & 96.5 & 92.4 &     & 83.7 & 
90.8 & 81.7 &      & 98.0 & 99.3 & 98.4 \\ 
			\bottomrule
		\end{tabular}%
  \end{adjustbox}	
  \end{center}
	\caption{\textbf{Clustering performance metrics on small-scale benchmark datasets, evaluated on their validation splits.} Probing means training a linear layer on top of the pretrained backbone in a supervised manner. We only highlight the best self-supervised pretrained model as the new state-of-the-art. We clarify that methods with $^{\dagger}$ use models pretrained on external data, while $^{\star}$ indicates methods that include additional dataset splits during training (i.e. validation data).}
    \label{tabel:small-datasets-sota}
\end{table*}

%% file: tables/cifar100_only.tex
\begin{table}[h]
\begin{center}
  
		\begin{tabular}{ l c c c }
			\toprule
   \multicolumn{1}{c}{Methods} & NMI(\%)  & ACC(\%) & ARI(\%) \\
\hline    
	TEMI DINO ViT-B/16 &  \textbf{76.9$\pm$0.45} &  \textbf{67.1$\pm$1.30} &  \textbf{53.3$\pm$1.02}   \\    	
    TEMI MSN ViT-L/16    &73.0$\pm$0.20 &  61.4$\pm$0.16 &  47.4$\pm$0.42  \\
\hline 
   \multicolumn{4}{l}{\textit{(natural language) supervised pretraining}}\\

TEMI CLIP ViT-L/14 &  79.9$\pm$0.23 &  73.7$\pm$0.92 &  61.2$\pm$0.75       \\
TEMI Sup. ViT-L/16  &85.2$\pm$0.34 &  81.8$\pm$0.73 &  70.6$\pm$0.89  \\
		\hline \hline
   \multicolumn{4}{l}{\textit{supervised baselines}}\\
  \multicolumn{1}{l}{Probing DINO ViT-B/16 }      & 85.7 & 85.3 & 73.6  \\ 

    \multicolumn{1}{l}{Probing MSN ViT-L/16}   & 84.6 & 84.4 & 71.9  \\

    \multicolumn{1}{l}{Probing CLIP ViT-L/14}        & 87.4 & 87.1 & 76.5  \\ 
     
   \multicolumn{1}{l}{Probing Sup. ViT-L/16}      & 86.0 & 
86.3 & 75.0     \\ 
			\bottomrule
		\end{tabular}%
	\end{center}	
	\caption{\textbf{Clustering performance metrics on on the CIFAR100 dataset.} All methods use models pretrained on external data.}
    \label{tabel:cifar100-only}
\end{table}

%% file: tables/sota_in_subs-comparison.tex
\begin{table*}
\begin{adjustbox}{width=0.95\columnwidth,center}
		\begin{tabular}{ l c c c c c c c c c c c }
			\toprule
			\multicolumn{1}{l}{Datasets} & \multicolumn{3}{c}{ImageNet 50} &      & \multicolumn{3}{c}{ImageNet 100} &      & \multicolumn{3}{c}{ImageNet 200} \\
			\cmidrule{2-4}\cmidrule{6-8}\cmidrule{10-12}    \multicolumn{1}{l}{Methods} & NMI(\%)  & ACC(\%) & ARI(\%) &      & NMI(\%) & ACC(\%) & ARI(\%) &      & NMI(\%) &ACC(\%)& ARI(\%)\\
				\hline
				
			SCAN (Resnet50) & 82.2 & 76.8 & 66.1 &      & 80.8 & 68.9 & 57.6 &      & 77.2 & 58.1 &	47.0 \\
  
  Propos (Resnet50)  & 82.8 & - & 69.1 &      &  83.5 & - & 63.5 &      & 80.6 & - & 53.8 \\

	\hline
TEMI DINO ViT-B/16 & 86.10$\pm$0.54 &  80.01$\pm$1.26 &  70.93$\pm$1.24 &      & 85.65$\pm$0.30 &  75.05$\pm$1.11 &  65.45$\pm$1.11 &      & 85.20$\pm$0.21 &  73.12$\pm$0.72 &  62.13$\pm$0.59 \\

TEMI MSN ViT-L/16 & \textbf{88.14$\pm$0.55} &  \textbf{84.87$\pm$1.16} &  \textbf{76.46$\pm$1.17} &      & \textbf{88.53$\pm$0.56} &  \textbf{82.86$\pm$0.73} &  \textbf{74.08$\pm$1.20} &      & \textbf{86.65$\pm$0.32} &  \textbf{77.96$\pm$0.71} &  \textbf{66.70$\pm$0.71} \\

\hline 
   \multicolumn{12}{l}{\textit{(natural language) supervised pretraining}}\\

TEMI CLIP ViT-L/14 & 92.32$\pm$0.38 &  88.27$\pm$0.53 &  82.78$\pm$0.94 &      & 90.06$\pm$0.53 &  83.43$\pm$1.98 &  75.81$\pm$1.36 &      & 88.39$\pm$0.16 &  77.76$\pm$0.37 &  69.41$\pm$0.23 \\

TEMI Sup. ViT-L/16 & 95.75$\pm$0.60 &  95.12$\pm$1.61 &  91.40$\pm$1.88 &      & 94.95$\pm$0.21 &  92.50$\pm$0.23 &  87.95$\pm$0.31 &      & 93.94$\pm$0.02 &  90.37$\pm$0.14 &  84.05$\pm$0.09 \\
	
  \hline
  \hline
     \multicolumn{12}{l}{\textit{supervised baselines}}\\
  
  \multicolumn{1}{l}{Probing DINO ViT-B/16 }  & 95.10 & 95.76 & 91.64 &     & 93.29 & 92.74 & 86.30 &      & 91.64 & 89.48 & 80.61 \\

  \multicolumn{1}{l}{Probing MSN ViT-L/16} & 94.21 & 94.92 & 90.03 &     & 93.00 & 92.42 & 85.74 &      & 91.36 & 89.02 & 79.88 \\
 
 \multicolumn{1}{l}{Probing CLIP ViT-L/14} 
   & 98.72 & 98.96 & 97.88 &     & 96.61 & 
96.16 & 92.73 &      & 95.09 & 93.57 & 88.00 \\

 \multicolumn{1}{l}{Probing Sup. ViT-L/16 }  & 97.77 & 98.12 & 96.21 &     & 96.13 & 95.76 & 91.90 &      & 95.07 & 93.60 & 88.02 \\
			\bottomrule
   \vspace{0.5cm}
   \end{tabular}
	\end{adjustbox}	
  \caption{\textbf{Clustering performances on ImageNet subsets.} All subsets were evaluated on their respective validation splits, as detailed in \Cref{tab:datasets-info}.}
    \label{tabel:imagenet-subsets}
\end{table*}

%% file: proofs.tex
We will first show that 
\begin{equation}
    \label{eq:thm-objective}
    \popt(c|x) = \arg\max_{q(c|x)} \Expected{x,x' \sim p(x,x')}{\pmi(x,x')},
\end{equation}

is bounded and leads to the correct joint distribution.

\begin{lemma}
\label{lemma:expected_pmi}
The mutual information
\begin{equation*}
    \MI{x}{x'} = \int p(x,x') \log \frac{p(x,x')}{p(x)p(x')} \dd x \dd x'
\end{equation*}
is an upper bound for the expected pointwise mutual information. In particular,
\begin{equation*}
\label{eq:expected_pmi}
\begin{aligned}
    & \Expected{x,x' \sim p(x,x')}{\pmi(x,x')} \\
    =& \MI{x}{x'} - \KL{p(x,x')}{\pthet(x,x'}),
\end{aligned}
\end{equation*}
where $\operatorname{KL}$ is the Kullback--Leibler divergence.  
\end{lemma}

\begin{proof}
\begin{align*}
    & && \Expected{x,x' \sim p(x,x')}{\pmi(x,x')}\\
    &=&& \int p(x,x') \log \frac{\pthet(x,x')}{p(x)p(x')} \dd x \dd x'\\
    &=&& \int p(x,x') \log \left(\frac{p(x,x')}{p(x)p(x')}
                            \frac{\pthet(x,x')}{p(x,x')} \right) \dd x \dd x' \\
    &=&& \int p(x,x') \log \frac{p(x,x')}{p(x)p(x')} \dd x \dd x' \\
       & &-&\int p(x,x') \log \frac{p(x,x')}{\pthet(x,x')} \dd x \dd x' \\
    &=&& \MI{x}{x'} - \KL{p(x,x')}{\pthet(x,x')}.
\end{align*}
\end{proof}

For the proof of Theorem 1, we now assume that each example $x$ belongs to exactly one cluster $c$ and need to show that the model $\popt(c|x)$ maximizing the objective
\begin{equation*}
    \popt(c|x) = \arg\max_{\pthet(c|x)} \Expected{x,x' \sim p(x,x')}{\pmi(x,x')}
\end{equation*}
is equal to $p(c|x)$ up to a permutation of the clusters.
\begin{proof}
Since $p(c|x)$ is one-hot by assumption, let us denote the class to which an image $x$ belongs as $c_x$, so that we have
\begin{equation*}
    p(c|x) = [c=c_x],
\end{equation*}
where the Iverson bracket $[c=c_x]$ is $1$ if $c=c_x$ and $0$ otherwise.
We denote the prediction of the classifier by $\hat{c}_x$, with
\begin{equation*}
    \hat{c}_x \coloneqq \operatorname*{argmax}_c \popt(c|x).
\end{equation*}
An equivalent formulation of the theorem then is:
$\popt(c|x)$ is one-hot for every $x$ and ${\hat{c}_x = \hat{c}_{x'}}$ if and only if ${c_x = c_{x'}}$.

Using the same factorization we used to formulate the pointwise mutual information in %
Equation (2) we obtain
\begin{equation*}
    \frac{p(x,x')}{p(x)p(x')} = \sum_{c=1}^C \frac{p(c|x)p(c|x')}{p(c)} = 
    [c_x = c_{x'}]p(c_x)^{-1}.
\end{equation*}

\Cref{lemma:expected_pmi} already states that the objective is maximized if and only if $\popt(x,x') = p(x,x')$
and therefore
\begin{align*}
&\pmi(x,x') = \sum_{c=1}^C \frac{\popt(c|x)\popt(c|x')}{\popt(c)} \\
& = \frac{\popt(x,x')}{p(x)p(x')}
= \frac{p(x,x')}{p(x)p(x')}
= [c_x = c_{x'}]p(c_x)^{-1}.
\end{align*}

If $c_x \ne c_{x'}$, we have
\begin{equation*}
0 \le \popt(\hat{c}_x|x)\popt(\hat{c}_x|x')/\popt(\hat{c}_x) \le \pmi(x,x') = 0.
\end{equation*}
Since $\popt(\hat{c}_x| x) > 0$, this implies
$\popt(\hat{c}_x|x')=0$ and therefore
$\hat{c}_x \ne \hat{c}_{x'}$.
Furthermore, from the pigeonhole principle it follows that $\popt(c|x) = 0$ for $c \ne c_x$ which both implies that $\popt(c|x)$ is one-hot as well as $\hat{c_x}=\hat{c}_{x'}$ if $c_x = c_{x'}$, therefore concluding the proof.
\end{proof}

%% file: tables/sup-data-info.tex
\begin{table}[t]
\begin{center}
\begin{tabular}{lrrrc}
\toprule
\textbf{Dataset} & \textbf{Classes} & \textbf{Train images} & \textbf{Val images} & \textbf{Size}\\ 
\midrule
CIFAR10 & 10 & 50,000 & 10,000 & 32 $\times$ 32\\
CIFAR100 & 100 & 50,000 & 10,000 & 32 $\times$ 32\\
CIFAR20 & 20 & 50,000 & 10,000 & 32 $\times$ 32\\
STL10 & 10 & 5,000 & 8,000 & 96 $\times$ 96\\
ImageNet-50 & 50 & 64,274 & 2,500 & 224 $\times$ 224\\
ImageNet-100 & 100 & 128,545 & 5,000 & 224 $\times$ 224\\
ImageNet-200 & 200 & 256,558 & 10,000 & 224 $\times$ 224\\
ImageNet & 1000 & 1,281,167 & 50,000 & 224 $\times$ 224\\
\bottomrule
\end{tabular}
\end{center}
\caption{\textbf{An overview of the number of classes and the number of samples on the considered datasets.} The train set is used for training, while the validation split is used to compute the clustering performance metrics. The selected classes on the ImageNet \cite{deng2009imagenet} subsets (ImageNet-50, ImageNet-100, and ImageNet-200) can be found in SCAN \cite{scan}.}
\label{tab:datasets-info}
\end{table}

%% file: tables/hp_pretrain.tex
\begin{table}[b]
\begin{center}
\label{tab:pretrain-hp} 
    \begin{tabular}{cc}
\toprule
config & value \\
\hline
optimizer & AdamW \\
base learning rate & $10^{-4}$ \\
weight decay & $10^{-4}$ \\
optimizer momentum & $\beta_1, \beta_2{=}0.9, 0.999$ \\
batch size & 512, 1024 (ImageNet)  \\
learning rate schedule & constant \\
softmax temperature $\tau$ & 0.1 \\
$\beta$ & 0.6, 0.55 (CIFAR20) \\
cluster heads & 50 \\
warmup epochs & 20, 10 ImageNet \\
training epochs & 200, 800 (STL10) \\
teacher momentum & 0.996 \\
augmentation & None \\
\bottomrule
\end{tabular}%
\end{center}
\caption{\textbf{Hyperparameters for training the clustering heads.}}
\end{table}

%% file: tables/hp_finetune.tex
\begin{table}[b]
\label{tab:finetune-hp} 
\begin{center}
\begin{tabular}{cc}
\toprule
config & value \\
\hline
optimizer & Adam \\
learning rate & $10^{-3}$ \\
weight decay & $10^{-3}$ \\
optimizer momentum & $\beta_1, \beta_2{=}0.9, 0.999$ \\
batch size & 256 \\
learning rate schedule & cosine decay \\
training epochs & 100 \\
augmentation & None \\
\bottomrule
\end{tabular}
\end{center}
\caption{\textbf{Hyperparameters for linear probing.}}
\end{table}

%% file: tables/sup-additional-results-sota.tex
\begin{table}[h]
\begin{center}
\begin{tabular}{lcccc}
			\toprule
			\multicolumn{1}{c}{Datasets} & \multicolumn{2}{c}{ImageNet} &     \multicolumn{2}{c}{CIFAR100}  \\
			\cmidrule{1-2} \cmidrule{3-5}    
   \multicolumn{1}{c}{Methods} &  TEMI  & \textit{k}-means  &       TEMI  & \textit{k}-means  \\
				\hline
		\multicolumn{5}{l}{\textit{self-supervised methods}}\\		
      MAE ViT-B/16    & 9.09$\pm$0.05 & 4.93 & 7.78$\pm$0.10 & 7.11 \\
    MAE ViT-L/16    & 27.81$\pm$0.13 & 12.45 & 19.56$\pm$0.17 & 12.05 \\
    MAE ViT-H/16    & 22.34$\pm$0.11 & 10.18 & 17.64$\pm$0.19 & 11.31 \\
 \hline
     MOCOv3 ViT-S/16 & 16.73$\pm$0.19 & 12.23 & 16.58$\pm$0.16 & 13.63 \\
    MOCOv3 ViT-B/16 & 54.10$\pm$0.08 & 47.64 & 63.51$\pm$0.53 & 49.94 \\
 \hline
	DINO Resnet50   & 45.20$\pm$0.23  & 32.07 & 45.34$\pm$0.41 & 34.21  \\
	DINO ViT-S/16   & 56.84$\pm$0.25 & 51.84  & 61.69$\pm$0.75 &  50.17 \\
	DINO ViT-B/16   & 58.08$\pm$0.26 & 52.26  & \textbf{67.11$\pm$1.30} & \textbf{57.01} \\
 \hline

    MSN ViT-S/16    & 58.53$\pm$0.39 & 55.58  & 63.06$\pm$0.89 & 49.96 \\
    MSN ViT-B/16    & 60.82$\pm$0.06 & 57.56  & 65.57$\pm$1.23 & 50.60 \\
    MSN ViT-L/16    & \textbf{61.56$\pm$0.28} & \textbf{58.08}  & 61.40$\pm$0.15 & 54.08 \\
 \hline
  \hline
  \multicolumn{5}{l}{\textit{natural language supervised methods}}\\
    CLIP Resnet50   & 45.93$\pm$0.11 & 34.41 & 34.06$\pm$0.72 & 25.96  \\ 
	CLIP ViT-B/16   & 56.68$\pm$0.24 & 45.86 & 60.74$\pm$0.79 & 45.84 \\ 
	CLIP ViT-L/14   & \textbf{63.99$\pm$0.38} & \textbf{54.12} & \textbf{73.70$\pm$0.92} & \textbf{54.55} \\
 \hline
  \hline
  \multicolumn{5}{l}{\textit{supervised methods}}\\	
   Resnet50   & 72.60$\pm$0.18 & 65.69 & 49.77$\pm$0.43 & 40.28 \\
    ConvNext S      & 77.67$\pm$0.41 & 71.85 & 57.31$\pm$0.20 & 43.19\\
   ConvNext B      & 78.23$\pm$0.12 & 73.67 & 58.31$\pm$0.76 & 43.20 \\
    ConvNext L      & \textbf{79.77$\pm$0.20} & \textbf{76.98} & 59.43$\pm$0.24 & 47.94 \\
   ViT-S/16   & 64.72$\pm$0.14 & 60.32 & 60.60$\pm$0.97 & 50.65 \\ 
	 ViT-B/16   & 69.23$\pm$0.27 & 64.48 & 63.36$\pm$0.43 & 51.72 \\
	ViT-L/16   & 77.12$\pm$0.21 & 74.91 & \textbf{81.77$\pm$0.73} & \textbf{70.06} \\

			\bottomrule
		\end{tabular}%
		\label{tab:addlabel}%
			\label{table:sota-comparison}
	\newline \newline \caption{\textbf{Benchmarking various models with the introduced objective versus \textit{k}-means.} We report the clustering accuracy (ACC) in \%}
\end{center}
\end{table}
	